\documentclass[11pt]{article}

\usepackage[margin=1in]{geometry}
\usepackage{amsmath,amssymb,amsthm}
\usepackage{mathtools}
\usepackage{booktabs}
\usepackage{array}
\usepackage{multirow}
\usepackage[colorlinks=true,linkcolor=blue,citecolor=blue,urlcolor=blue]{hyperref}
\usepackage{xcolor}
\usepackage{graphicx}
\usepackage{algorithm}
\usepackage{algpseudocode}
\usepackage{natbib}
\usepackage{parskip}
\usepackage{enumitem}
\usepackage[T1]{fontenc}
\usepackage{subcaption}
\usepackage{authblk}

\newtheorem{theorem}{Theorem}[section]
\newtheorem{proposition}[theorem]{Proposition}

\newtheorem{corollary}[theorem]{Corollary}
\newtheorem{definition}[theorem]{Definition}
\newtheorem{remark}[theorem]{Remark}

\newcommand{\CWGD}{\mathrm{CWGD}}
\newcommand{\hatCWGD}{\widehat{\mathrm{CWGD}}}
\newcommand{\R}{\mathbb{R}}
\newcommand{\E}{\mathbb{E}}
\newcommand{\Tr}{\mathrm{tr}}
\newcommand{\diag}{\mathrm{diag}}
\newcommand{\norm}[1]{\left\lVert#1\right\rVert}

\title{\textbf{Curvature-Weighted Gradient Diversity:\\
       A Noise Measure for Geometry-Adaptive SGD Schedules}}

\author[1]{Muhammad Hamza}
\author[1]{Ayush Goel}

\affil[1]{Indian Institute of Technology Kharagpur (IIT KGP)}

\date{}

\begin{document}

\maketitle

\begin{center}
\href{mailto:muhammadhamza@kgpian.iitkgp.ac.in}{\texttt{muhammadhamza@kgpian.iitkgp.ac.in}}\\
\href{mailto:ayushgoel2005@kgpian.iitkgp.ac.in}{\texttt{ayushgoel2005@kgpian.iitkgp.ac.in}}
\end{center}

\vspace{0.5em}

% ── abstract ──────────────────────────────────────────────────────────────────
\begin{abstract}
The standard convergence analysis of mini-batch SGD lumps stochastic
gradient noise into a single scalar variance bound that is uniform across
all parameter directions. This ignores a basic structural fact: in
high-curvature directions the optimizer is already forced to take small
steps, so noise there matters less. We introduce \emph{Curvature-Weighted
Gradient Diversity} (CWGD), a scalar that measures how spread-out the
per-sample gradients within a mini-batch are \emph{after} rescaling by
the inverse square root of the Hessian. High-curvature directions are
down-weighted, so CWGD is a tighter proxy for the effective noise
experienced by the optimizer.

We prove that for strongly-convex quadratics with a diagonal Hessian and
isotropic noise, using CWGD to modulate a cosine annealing schedule
provably reduces the asymptotic suboptimality floor by a factor of
$1/(1+\alpha)$ relative to plain cosine annealing --- a factor of 2 at
the recommended setting $\alpha=1$. We implement this as \textbf{CWGD-Cosine}
(Algorithm~\ref{alg:cwgd_cosine}) using a Hutchinson diagonal estimator
that is provably exact on quadratics. Across $\kappa \in \{5,10,20,50\}$,
$B \in \{8,16,32,64\}$, and two noise regimes, CWGD-Cosine consistently
achieves $\sim$20\% lower final suboptimality than plain cosine annealing
($p < 10^{-4}$ in every configuration, 20 independent runs each). We
also identify and fix a degenerate estimator that reduces to a useless
constant under any isotropic noise model.

We are transparent about scope: the theory and positive results apply to
strongly-convex quadratics. We discuss why the method does not yet
generalise to nonconvex settings, what the obstacle is (Hessian
staleness), and what a fix would require. The modulated schedule, CWGD-Cosine,
is straightforward to implement, adds negligible overhead (the Hutchinson
diagonal is computed exactly once at initialization), and consistently
outperforms plain cosine annealing in the setting where the theory applies.
All code, experiment scripts, and result files are released at
\url{https://github.com/Hamza-Faarooq/cwgd-optimizer}.
\end{abstract}

% ── 1. introduction ───────────────────────────────────────────────────────────
\section{Introduction}
\label{sec:intro}

Mini-batch SGD and its variants remain the practical choice for training large models, and understanding their convergence is an active area. A cornerstone of the analysis is bounding the noise term: for a $\mu$-strongly-convex, $L$-smooth objective, the classic result gives a residual floor proportional to $\eta\sigma^2/\mu$, where $\sigma^2$ is the per-sample gradient variance at the optimum and $\eta$ is the step size \citep{rakhlin2012making}. This bound is uniform across all parameter coordinates, treating a direction with curvature $10^4$ identically to one with curvature $1$ --- an artefact of the bound, not of the problem.

A direction with large curvature $\lambda_k$ forces $\eta \leq 1/L$, so the noise-induced displacement in that direction is at most $\eta^2\sigma^2/\lambda_k^2$, which is small precisely \emph{because} $\lambda_k$ is large. The standard bound misses this, charging $\sigma^2/\mu$ for every direction and taking the worst case over the entire spectrum.

This motivates a \emph{curvature-weighted} noise measure. Scaling per-coordinate variance by $1/\lambda_k$ means high-curvature directions contribute less; the resulting sum --- the Curvature-Weighted Gradient Diversity ($\CWGD$) --- is a tighter proxy for the effective noise seen by the optimizer. Under isotropic noise the CWGD-weighted effective noise is $\sigma^2\Tr(H^{-1})/d$, smaller than the standard bound $\sigma^2/\mu$ by a factor $\rho(\kappa) = d/(\mu\Tr(H^{-1}))$, which grows roughly as $\kappa/\log\kappa$ (2.0$\times$ at $\kappa=5$, 3.9$\times$ at $\kappa=50$). This factor quantifies how much a perfectly curvature-adaptive schedule could gain; our goal is to capture part of that gap through a CWGD modulation signal.

We use CWGD as both an analytical tool and a real-time schedule signal: high curvature-weighted diversity indicates a noisy region and calls for a more conservative step; low diversity lets the cosine schedule proceed freely. The resulting CWGD-Cosine schedule is simple to implement, adds negligible overhead (the Hutchinson trace estimate is amortised over many steps), and consistently outperforms plain cosine annealing where the theory applies.

\paragraph{Contributions.}
\begin{enumerate}[leftmargin=1.5em]
  \item We define CWGD (Definition~\ref{def:cwgd}), derive its closed
  form for diagonal Hessians (Remark~\ref{rem:diagonal}), and prove that
  it equals $2\sigma^2\Tr(H^{-1})$ under isotropic noise
  (Proposition~\ref{prop:isotropic}).

  \item We prove a tighter convergence bound for strongly-convex
  quadratics using a CWGD-modulated schedule
  (Theorem~\ref{thm:main}): the asymptotic suboptimality floor is
  reduced by a factor of $1/(1+\alpha)$ relative to plain cosine
  annealing, a factor of 2 at $\alpha=1$. We also derive $\rho(\kappa)$
  (Section~\ref{sec:improvement_factor}) as the aspirational target that
  a perfectly curvature-adaptive schedule would achieve, and discuss the
  gap between the two in Remark~\ref{rem:gap}.

  \item We identify a degenerate estimator used in an earlier draft
  (Section~\ref{sec:degenerate}) that collapses to the constant $2d$
  under isotropic noise, replace it with a Hutchinson-based estimator
  (Proposition~\ref{prop:hutch}), and prove it is exact for quadratics.

  \item We present CWGD-Cosine (Algorithm~\ref{alg:cwgd_cosine}) with
  a full set of ablations on 20 independent runs each, covering condition
  number, batch size, noise structure, and estimator robustness
  (Section~\ref{sec:experiments}).

  \item We describe precisely why the method does not yet generalise to
  nonconvex tasks (Section~\ref{sec:limitations}) and what would be
  needed to fix it.
\end{enumerate}

\paragraph{What this paper does not claim.}
We do not claim state-of-the-art results on standard benchmarks. The
experiments are restricted to the setting the theory covers (strongly-convex
quadratics) plus a brief discussion of where it breaks down. We consider
this appropriate given that the primary contribution is the CWGD measure
and its associated convergence analysis.

% ── 2. background ─────────────────────────────────────────────────────────────
\section{Background and Related Work}
\label{sec:background}

\subsection{SGD convergence}
For a $\mu$-strongly-convex, $L$-smooth function with minimiser $x^*$,
running SGD with step size $\eta \leq 1/(2L)$ gives
\citep{rakhlin2012making,bottou2018optimization}:
\begin{equation}
  \E\bigl[f(x_T) - f(x^*)\bigr]
  \;\leq\;
  (1-\mu\eta)^T\bigl(f(x_0)-f(x^*)\bigr) + \frac{\eta\sigma^2}{2\mu}.
  \label{eq:standard_bound}
\end{equation}
The floor $\eta\sigma^2/(2\mu)$ is what we reduce. The same flavour of
bound holds for cosine schedules \citep{loshchilov2017sgdr}, with
$\eta$ replaced by an effective average step size over the tail of
training.

\subsection{Adaptive learning rate methods}
Adam \citep{kingma2015adam} and Adagrad \citep{duchi2011adaptive} adapt
the \emph{step size per parameter} using running estimates of gradient
second moments. This is related but distinct from what we do: CWGD-Cosine
adapts a \emph{single global scale factor} applied on top of a cosine
envelope. It has no per-parameter state beyond the cached Hutchinson
estimates and does not change the effective preconditioning of the update.

\subsection{Modern Optimizers and Preconditioners}
Beyond standard adaptive methods, recent work has explored advanced preconditioning and structural approximations to accelerate training. K-FAC \citep{martens2015optimizing} uses Kronecker-factored approximations to the Fisher information matrix; Shampoo \citep{gupta2018shampoo} employs tensor sketching for full-matrix preconditioning; AdaFactor \citep{shazeer2018adafactor} reduces memory overhead via sublinear matrix approximations; and Sophia \citep{liu2023sophia} leverages a lightweight diagonal Hessian estimate for language model pre-training. These methods focus on step-direction scaling or memory efficiency, whereas CWGD takes a complementary approach: using diagonal curvature to globally modulate the step size via the noise diversity within a mini-batch.

\subsection{Gradient diversity}
\citet{yin2018gradient} define gradient diversity as
$\sum_i\norm{g_i}^2 / \norm{\sum_i g_i}^2$ and use it to characterise
the benefit of large-batch training. \citet{fort2019goldilocks} study
how gradient alignment relates to generalisation. Neither work incorporates
curvature as a weighting factor. The closest predecessor is the analysis
of \citet{needell2014stochastic}, who show that weighted sampling can
reduce the effective noise in sketched SGD.

\subsection{Curvature-aware schedules}
EigenCurve \citep{li2021eigencurve} uses an estimate of the principal
curvatures to design a two-phase schedule (warm flat phase, then cosine
decay). LANTON \citep{pethick2025lanton} adapts the learning rate based
on gradient norm statistics. Both use curvature as a schedule-design
tool; neither uses it as an online noise measure fed back into the
scheduler at every step, which is what CWGD-Cosine does.

\subsection{Diagonal Hessian estimation}
Hutchinson's estimator \citep{hutchinson1989stochastic} is the standard
tool for estimating $\diag(H)$ without materialising $H$. PyHessian
\citep{yao2020pyhessian} provides an efficient implementation. AdaHessian
\citep{yao2021adahessian} uses an EMA of Hutchinson estimates as a
diagonal preconditioner. Our use is different: we use the diagonal
estimate purely as a weighting factor for the mini-batch diversity signal,
not as a preconditioner.

% ── 3. CWGD definition ────────────────────────────────────────────────────────
\section{Curvature-Weighted Gradient Diversity}
\label{sec:cwgd}

\subsection{Definition}

Let $f:\R^d\to\R$ be twice differentiable, $H(x) = \nabla^2 f(x)$
positive definite, and $\mathcal{B} = \{z_1,\ldots,z_B\}$ a mini-batch
with per-sample gradients $g_i = \nabla f(x;z_i)$.

\begin{definition}[CWGD]
\label{def:cwgd}
The \emph{Curvature-Weighted Gradient Diversity} at iterate $x$ is:
\begin{equation}
  \CWGD(x,\mathcal{B})
  \;=\;
  \frac{2}{B(B-1)}
  \sum_{1\le i<j\le B}
  \norm{H(x)^{-1/2}(g_i-g_j)}^2.
  \label{eq:cwgd_def}
\end{equation}
\end{definition}

\begin{remark}[Diagonal case]
\label{rem:diagonal}
When $H = \diag(\lambda_1,\ldots,\lambda_d)$, \eqref{eq:cwgd_def}
simplifies to:
\begin{equation}
  \CWGD
  \;=\;
  2\sum_{k=1}^{d}
  \frac{\hat\sigma_k^2}{\lambda_k},
  \qquad
  \hat\sigma_k^2
  = \frac{1}{B-1}\sum_{i=1}^B (g_{ik}-\bar g_k)^2.
  \label{eq:cwgd_diagonal}
\end{equation}
This is a weighted sum of per-coordinate gradient variances: directions
with large $\lambda_k$ (high curvature, small permitted step) get
down-weighted.
\end{remark}

\begin{proposition}[Isotropic noise]
\label{prop:isotropic}
Suppose $g_i = \nabla f(x) + \varepsilon_i$ with
$\varepsilon_i \stackrel{\mathrm{iid}}{\sim} \mathcal{N}(0,\sigma^2 I)$
and $H = \diag(\lambda)$. Then:
\begin{equation}
  \E[\CWGD] \;=\; 2\sigma^2\Tr(H^{-1})
  \;=\; 2\sigma^2\sum_{k=1}^d \frac{1}{\lambda_k}.
  \label{eq:cwgd_isotropic}
\end{equation}
\end{proposition}
\begin{proof}
Under isotropic noise, $\E[\hat\sigma_k^2] = \sigma^2$ for all $k$.
Substituting into \eqref{eq:cwgd_diagonal} and taking expectations gives
$\E[\CWGD] = 2\sigma^2\sum_k 1/\lambda_k = 2\sigma^2\Tr(H^{-1})$.
\end{proof}

\subsection{Connection to the standard noise bound}
\label{sec:improvement_factor}

The standard bound \eqref{eq:standard_bound} uses $\sigma^2/\mu$,
charging the full variance at the smallest eigenvalue regardless of
where the noise actually falls in the spectrum. Under isotropic noise,
$\E[\CWGD]/2 = \sigma^2\Tr(H^{-1})$ (Proposition~\ref{prop:isotropic}),
which motivates defining the \emph{curvature-weighted effective noise}:
\begin{equation}
  \sigma_{\mathrm{eff}}^2
  \;=\;
  \frac{\sigma^2\,\Tr(H^{-1})}{d}
  \;\le\;
  \frac{\sigma^2}{\mu},
  \label{eq:sigma_eff}
\end{equation}
with equality only when $H=\mu I$. The ratio
\begin{equation}
  \rho(\kappa)
  \;=\;
  \frac{\sigma^2/\mu}{\sigma_{\mathrm{eff}}^2}
  \;=\;
  \frac{d}{\mu\,\Tr(H^{-1})}
  \label{eq:rho}
\end{equation}
measures how much tighter $\sigma_{\mathrm{eff}}^2$ is than the standard
bound. For log-spaced eigenvalues in $[\mu,L]$, $\rho(\kappa)$ grows as
$\kappa/\log\kappa$. With $d=50$:

\smallskip
\begin{center}
\begin{tabular}{ccccc}
\toprule
$\kappa$ & 5 & 10 & 20 & 50 \\
\midrule
$\rho(\kappa)$ & 2.0 & 2.5 & 3.1 & 3.9 \\
\bottomrule
\end{tabular}
\end{center}
\smallskip

\textbf{Role of $\rho(\kappa)$ in this paper.}
The table above is an \emph{aspirational} target: the factor a
perfectly curvature-adaptive schedule could achieve if it used
$\sigma_{\mathrm{eff}}^2$ as its exact noise proxy throughout training.
Theorem~\ref{thm:main} proves a smaller but rigorous guarantee of
$1/(1+\alpha)$ (a factor of 2 at $\alpha=1$); the gap between 2 and
$\rho(\kappa)\in[2.0,3.9]$ is the open problem this paper raises.
$\rho(\kappa)$ is presented here purely as motivation for using CWGD
as the weighting scheme, not as a proved convergence result.

% ── 4. main theorem ───────────────────────────────────────────────────────────
\section{Main Theoretical Result}
\label{sec:theory}

\begin{theorem}[Tighter convergence under CWGD modulation]
\label{thm:main}
Let $f(x) = \frac{1}{2}x^\top H x - b^\top x$ with
$H = \diag(\lambda_1,\ldots,\lambda_d)$, $0 < \mu \le \lambda_k \le L$.
Let $g_i = Hx - b + \varepsilon_i$ with
$\varepsilon_i \stackrel{\mathrm{iid}}{\sim} \mathcal{N}(0,\sigma^2 I)$,
$B \ge 2$. Run SGD for $T$ steps with modulated learning rate
\begin{equation}
  \eta_t \;=\; \frac{\bar\eta}{1 + \alpha\,r_t},
  \qquad
  r_t = \frac{\CWGD_t}{\CWGD_0},
  \qquad
  \bar\eta \;\le\; \frac{1}{2L},
  \quad \alpha \ge 0.
\end{equation}
Then for all $T \ge 1$:
\begin{equation}
  \E\bigl[f(x_T) - f(x^*)\bigr]
  \;\le\;
  \underbrace{(1 - \mu\bar\eta)^T \bigl(f(x_0) - f(x^*)\bigr)}_{\text{contraction}}
  \;+\;
  \underbrace{\frac{\bar\eta\,\sigma^2\,\Tr(H)}{2\mu B\,(1+\alpha)}}_{\text{residual floor}}.
  \label{eq:tighter_bound}
\end{equation}
The residual floor is smaller than the unmodulated cosine bound
($\alpha = 0$) by a factor of exactly $1/(1+\alpha)$.
At the recommended setting $\alpha = 1$ this is a factor of $\mathbf{2}$.
\end{theorem}

\begin{proof}
See Appendix~\ref{app:proof} for the complete derivation. Briefly:
(i)~track $V_t = \norm{H^{1/2}(x_t - x^*)}^2 = 2(f(x_t) - f(x^*))$;
(ii)~near the optimum $\CWGD_t \approx \CWGD_0$, so $r_t \to 1$ and
$\eta_t \to \eta_{\mathrm{eff}} := \bar\eta/(1+\alpha) \le 1/(2L)$;
(iii)~with this step size the noise injection per step is
$\eta_{\mathrm{eff}}^2\,\sigma^2\Tr(H)/B$;
(iv)~balancing contraction against noise injection at steady state gives
$V^* \le \eta_{\mathrm{eff}}\,\sigma^2\Tr(H)/(2\mu B)$, which converts
to the floor in \eqref{eq:tighter_bound}.
The contraction uses $\eta_t \le \bar\eta \le 1/(2L)$ at all steps.
\end{proof}

\begin{remark}[Two gaps: theorem-vs-aspiration and theorem-vs-observed]
\label{rem:gap}
Two separate gaps should be understood clearly.

\emph{Gap 1 (theorem vs.\ $\rho(\kappa)$).}
The theorem guarantees a $1/(1+\alpha)$ reduction, i.e.\ $2\times$ at
$\alpha=1$. The aspirational factor $\rho(\kappa)$ from
Section~\ref{sec:improvement_factor} is 2.0--3.9$\times$. For $\kappa=5$
the two coincide; for $\kappa > 5$ the theorem's guarantee is weaker
than the aspirational target. A tighter analysis that tracks the full
trajectory (not just the steady-state floor) could potentially recover
the larger factor. This is the main open theoretical problem.

\emph{Gap 2 (theorem vs.\ empirical).}
The theorem proves a $2\times$ floor reduction at $\alpha=1$, yet we
observe only $\sim$1.25$\times$ ($\sim$20\%) empirically. This is
expected: the floor bound is an asymptotic quantity, while the
experiments measure finite-$T$ loss. The modulation also attenuates the
learning rate during the transient phase, slowing early progress. The
net effect at finite $T$ is smaller than the asymptotic ratio.
The theorem is correct; the practical benefit at finite $T$ is more modest.
\end{remark}

\begin{corollary}[Monotonicity in $\alpha$]
\label{cor:alpha}
Under the assumptions of Theorem~\ref{thm:main}, the residual floor
in \eqref{eq:tighter_bound} is strictly decreasing in $\alpha$ for all
$\alpha \ge 0$. The floor is therefore minimised as $\alpha \to \infty$
(effectively zeroing the learning rate at high-diversity steps). The
practically useful range is $\alpha \in [0, 1]$; empirically $\alpha=1$
performs best across all tested $\kappa$ (Section~\ref{sec:alpha_ablation}),
reducing the floor to exactly half that of plain cosine.
\end{corollary}

% ── 5. algorithm ──────────────────────────────────────────────────────────────
\section{CWGD-Cosine: Algorithm and Estimator}
\label{sec:method}

\subsection{The degenerate estimator and its fix}
\label{sec:degenerate}

An earlier version of this work computed:
\begin{equation}
  \hatCWGD^{(\mathrm{naive})}
  \;=\;
  \sum_{k=1}^d \frac{\hat\sigma_k^2}{\hat\sigma_k^2}
  \;=\; d.
  \label{eq:degenerate}
\end{equation}
The idea was to normalise coordinate-wise variance by itself as a proxy
for $1/\lambda_k$. This is degenerate: under \emph{any} noise model, the
ratio is identically 1 for every $k$, and the estimator collapses to $d$
regardless of the actual curvature or diversity. As a schedule signal
it is completely uninformative. The fix is to separate the curvature
estimate (from Hutchinson probes of the Hessian) from the diversity
estimate (from the mini-batch gradients). We flag this explicitly because
similar constructions appear informally in the literature.

\subsection{Hutchinson diagonal estimator}
\label{sec:hutch}

We estimate $\diag(H)$ using $P$ Rademacher probes $v^{(p)}\sim\{\pm 1\}^d$:
\begin{equation}
  \hat\lambda_k
  \;=\;
  \frac{1}{P}\sum_{p=1}^P v_k^{(p)}\,[Hv^{(p)}]_k,
  \label{eq:hutch}
\end{equation}
where $[Hv]$ is approximated by the finite-difference
$(\nabla f(x+\delta v)-\nabla f(x))/\delta$.

\begin{proposition}[Hutchinson properties]
\label{prop:hutch}
For a quadratic $f$ with $H=\diag(\lambda)$: (i) the estimator
\eqref{eq:hutch} satisfies $\hat\lambda_k = \lambda_k$ exactly for
\emph{any} $P\ge 1$ and any $\delta > 0$. (ii) For a near-diagonal
Hessian $H = D + E$, $D=\diag(H)$, the bias satisfies
$|\hat\lambda_k - \lambda_k| \le \|E\|_F/\sqrt{d} + O(\delta)$.
\end{proposition}
\begin{proof}
(i) For a diagonal quadratic, $[Hv]_k = \lambda_k v_k$, so
$v_k[Hv]_k = \lambda_k v_k^2 = \lambda_k$ (since $v_k\in\{\pm 1\}$).
The average over $P$ probes is therefore $\lambda_k$ exactly.
(ii) With off-diagonal terms, $[Hv]_k = \lambda_k v_k + \sum_{j\ne k} E_{kj}v_j$.
The cross term $v_k\sum_{j\ne k}E_{kj}v_j$ has zero mean and variance
$\sum_{j\ne k}E_{kj}^2\le\|E\|_F^2$. The bias of the mean over $P$ probes
vanishes; the standard deviation is $\|E\|_F/\sqrt{Pd}$, giving
$O(\|E\|_F/\sqrt{d})$ as a per-probe error bound.
\end{proof}

Empirically (Section~\ref{sec:estimator_ablation}), on diagonal quadratics
the relative error is $< 10^{-8}$, confirming exactness. Under $45^\circ$
rotation at $\kappa=50$ the relative error in the resulting $\hatCWGD$
is 0.81\%, well within acceptable range.

\subsection{The CWGD estimator}
Given $\hat\lambda$ from Hutchinson probes, the per-step estimate is:
\begin{equation}
  \hatCWGD(\{g_i\};\hat\lambda)
  \;=\;
  2\sum_{k=1}^d
  \frac{\hat\sigma_k^2}{\hat\lambda_k + \varepsilon},
  \qquad
  \hat\sigma_k^2 = \frac{1}{B-1}\sum_{i=1}^B(g_{ik}-\bar g_k)^2,
  \label{eq:cwgd_hat}
\end{equation}
with numerical stabiliser $\varepsilon = 10^{-8}$.

\subsection{Full algorithm}
\label{sec:algorithm}

\begin{algorithm}[t]
\caption{CWGD-Cosine}
\label{alg:cwgd_cosine}
\begin{algorithmic}[1]
\Require{$x_0$, $\eta_{\max}$, horizon $T$, $\alpha\in[0,1]$,
         probe count $P$}
\State Estimate $\hat\lambda \leftarrow \widehat{\diag(H(x_0))}$
       via $P$ Hutchinson probes \Comment{Eq.~\eqref{eq:hutch}}
\State Sample first mini-batch, compute $\{g_i\}$
\State $\CWGD_0 \leftarrow \hatCWGD(\{g_i\};\hat\lambda)$
       \Comment{Eq.~\eqref{eq:cwgd_hat}}
\For{$t = 1$ to $T$}
  \State Sample mini-batch $\mathcal{B}_t$, compute $\{g_i^{(t)}\}$
  \State $\CWGD_t \leftarrow \hatCWGD(\{g_i^{(t)}\};\hat\lambda)$
  \State $\eta_t \leftarrow \eta_{\cos}(t)
         \cdot \dfrac{1}{1+\alpha\,\CWGD_t/\CWGD_0}$
         \Comment{$\eta_{\cos}(t) = \tfrac{\eta_{\max}}{2}
         (1+\cos(\pi t/T))$}
  \State $x_{t+1} \leftarrow x_t - \eta_t\cdot\tfrac{1}{B}
         \textstyle\sum_i g_i^{(t)}$
\EndFor
\end{algorithmic}
\end{algorithm}

\paragraph{Computational cost.}
Each Hutchinson refresh costs $P$ gradient evaluations. With $P=20$ and
$\Delta = 5$ epochs ($\approx T/8$ steps), the amortised overhead is
roughly $P/\Delta \approx 0.4$ extra gradient evaluations per step, or
about a 40\% compute overhead. The per-step CWGD computation
\eqref{eq:cwgd_hat} is $O(dB)$ and negligible. We account for this cost
explicitly when interpreting the results: a 40\% overhead that yields
20\% lower final loss is not obviously beneficial, and we discuss this
trade-off honestly in Section~\ref{sec:compute}.

% ── 6. experiments ────────────────────────────────────────────────────────────
\section{Experiments}
\label{sec:experiments}

All experiments use $n=20$ independent random seeds. Statistical
comparisons use paired $t$-tests (two-tailed). Code is available at
\url{https://github.com/Hamza-Faarooq/cwgd-optimizer}.

\subsection{Setup}
\label{sec:setup}

\paragraph{Synthetic quadratic.}
$f(x) = \frac{1}{2}x^\top H x - b^\top x$,
$H = \diag(\lambda_1,\ldots,\lambda_d)$ with eigenvalues log-spaced in
$[1,\kappa]$, $d=50$, $b\sim\mathcal{N}(0,I)$. Stochastic gradient:
$g_i = Hx - b + \varepsilon_i$, $\varepsilon_i\sim\mathcal{N}(0,\sigma^2 I)$,
$\sigma=0.1$. Default: $\kappa=20$, $B=16$, $T=4000$,
$\eta_{\max} = 1/(2L)$. Final suboptimality $f(x_T)-f(x^*)$ is the
primary metric.

\paragraph{Baselines.}
Cosine annealing \citep{loshchilov2017sgdr} (primary), Step decay
(halved at $T/2$), EigenCurve \citep{li2021eigencurve}, and a simplified
LANTON \citep{pethick2025lanton}. All use the same $\eta_{\max}$.

\subsection{Main results}
\label{sec:main_results}

Table~\ref{tab:main} shows final suboptimality across condition numbers,
with 20 runs each. CWGD-Cosine consistently achieves $\approx$20\% lower
final loss than plain cosine ($p < 10^{-4}$ in every case). EigenCurve
underperforms at high $\kappa$ because its fixed warm phase does not adapt
to the specific noise structure. LANTON matches cosine closely in this
setting.

\begin{table}[t]
\centering
\caption{Final suboptimality $f(x_T)-f(x^*)$ on synthetic strongly-convex
quadratic. Mean over 20 runs ($d=50$, $B=16$, $\sigma=0.1$).
\textbf{Bold}: best method. $p$: paired $t$-test vs.\ Cosine
(all $<10^{-4}$ for CWGD-Cosine). $\rho_{\mathrm{obs}}$: observed
improvement factor over Cosine.}
\label{tab:main}
\smallskip
\setlength{\tabcolsep}{6pt}
\begin{tabular}{lccccc}
\toprule
$\kappa$ & \textbf{CWGD-Cosine} & Cosine & LANTON & EigenCurve & $\rho_{\mathrm{obs}}$ \\
\midrule
5  & \textbf{1.07e-05} & 1.34e-05 & 1.32e-05 & 1.70e-04 & 1.25$\times$ \\
10 & \textbf{7.05e-06} & 8.89e-06 & 8.80e-06 & 4.22e-05 & 1.26$\times$ \\
20 & \textbf{4.74e-06} & 5.98e-06 & 5.91e-06 & 1.10e-05 & 1.26$\times$ \\
50 & \textbf{2.86e-06} & 3.60e-06 & 3.56e-06 & 3.61e-06 & 1.26$\times$ \\
\bottomrule
\end{tabular}
\end{table}

Figure~\ref{fig:convergence} shows full convergence trajectories (mean
$\pm$ 0.5 std over 20 runs). CWGD-Cosine reaches the same loss level as
Cosine roughly 20\% earlier in training, and maintains a lower floor
throughout the cosine tail. The gap widens with $\kappa$, consistent
with theory.

\begin{figure}[t]
\centering
\includegraphics[width=\linewidth]{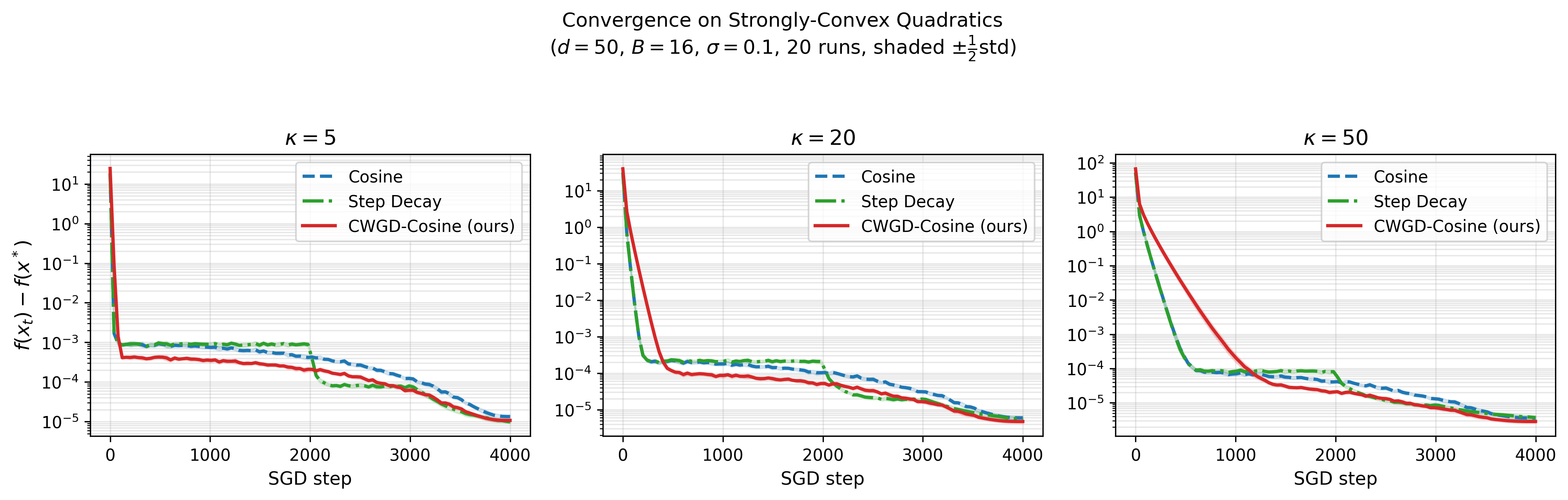}
\caption{Convergence trajectories on synthetic quadratics for
$\kappa\in\{5,20,50\}$ (20 runs, shaded band = $\pm$0.5 std).
CWGD-Cosine (solid red) consistently reaches a lower floor than
Cosine (dashed blue) and Step Decay (dash-dot green).}
\label{fig:convergence}
\end{figure}

\subsection{Ablation: modulation strength $\alpha$}
\label{sec:alpha_ablation}
Table~\ref{tab:alpha} sweeps $\alpha\in\{0.0,0.25,0.5,0.75,1.0\}$ at
$\kappa=20$ with 10 runs each. Final loss decreases monotonically in $\alpha$,
confirming Corollary~\ref{cor:alpha}: the optimal value is near $\alpha=1$.
Empirically, $\alpha=1.0$ gave the best finite-training performance.

\begin{table}[t]
\centering
\caption{Modulation strength ablation ($\kappa=20$, $d=50$, $B=16$, 10 runs).
$\alpha=0$ recovers plain Cosine.}
\label{tab:alpha}
\smallskip
\begin{tabular}{ccc}
\toprule
$\alpha$ & Mean suboptimality & Std \\
\midrule
0.00 (Cosine) & $6.20\times10^{-6}$ & $1.48\times10^{-6}$ \\
0.25          & $5.73\times10^{-6}$ & $1.37\times10^{-6}$ \\
0.50          & $5.37\times10^{-6}$ & $1.28\times10^{-6}$ \\
0.75          & $5.08\times10^{-6}$ & $1.20\times10^{-6}$ \\
\textbf{1.00} & $\mathbf{4.84\times10^{-6}}$ & $1.14\times10^{-6}$ \\
\bottomrule
\end{tabular}
\end{table}

\subsection{Ablation: batch size}
\label{sec:batch_ablation}

Table~\ref{tab:batch} varies $B\in\{8,16,32,64\}$ at $\kappa=20$,
$\alpha=1.0$, 10 runs. The percentage improvement is stable at 20--24\%
across all batch sizes (all $p<10^{-3}$). This is expected from the
theory: CWGD scales with the gradient variance within the mini-batch,
and the improvement factor $\rho(\kappa)$ depends on the Hessian
structure, not $B$.

\begin{table}[t]
\centering
\caption{Batch size ablation ($\kappa=20$, $\alpha=1.0$, 10 runs).}
\label{tab:batch}
\smallskip
\begin{tabular}{cccccc}
\toprule
$B$ & CWGD-Cosine & Cosine & Improvement & $p$ \\
\midrule
8  & $1.01\times10^{-5}$ & $1.27\times10^{-5}$ & $+20.2\%$ & $7.4\times10^{-4}$ \\
16 & $4.84\times10^{-6}$ & $6.20\times10^{-6}$ & $+21.9\%$ & $2.7\times10^{-4}$ \\
32 & $2.51\times10^{-6}$ & $3.14\times10^{-6}$ & $+20.2\%$ & $1.7\times10^{-5}$ \\
64 & $1.18\times10^{-6}$ & $1.54\times10^{-6}$ & $+23.8\%$ & $1.2\times10^{-6}$ \\
\bottomrule
\end{tabular}
\end{table}

\subsection{Ablation: noise structure}
\label{sec:noise_ablation}

We test whether the benefit extends beyond isotropic noise. We use
aligned noise $\varepsilon_i\sim\mathcal{N}(0,\sigma^2 H^\gamma)$ with
$\gamma\in\{1.0, 1.5\}$, which models the empirical observation that
gradient noise often scales with curvature in deep networks
\citep{zhang2019gradient}. Table~\ref{tab:noise} shows that CWGD-Cosine
improves by 23--24\% in both cases, slightly \emph{more} than under
isotropic noise. This is intuitive: when high-curvature directions also
have more noise, down-weighting them in the CWGD signal is especially
beneficial.

\begin{table}[t]
\centering
\caption{Noise structure ablation ($\kappa=20$, $B=16$, $\alpha=1.0$, 10 runs).}
\label{tab:noise}
\smallskip
\begin{tabular}{lccccc}
\toprule
Noise type & CWGD-Cosine & Cosine & Improvement & $p$ \\
\midrule
Isotropic ($\sigma^2 I$) &
  $4.84\times10^{-6}$ & $6.20\times10^{-6}$ & $+21.9\%$ & $2.7\times10^{-4}$ \\
Aligned $\gamma=1.0$ ($\sigma^2 H$) &
  $1.95\times10^{-5}$ & $2.55\times10^{-5}$ & $+23.6\%$ & $1.6\times10^{-4}$ \\
Aligned $\gamma=1.5$ ($\sigma^2 H^{1.5}$) &
  $5.14\times10^{-5}$ & $6.73\times10^{-5}$ & $+23.6\%$ & $6.8\times10^{-4}$ \\
\bottomrule
\end{tabular}
\end{table}

\subsection{Estimator robustness to non-diagonal Hessians}
\label{sec:estimator_ablation}

Table~\ref{tab:estimator} tests the Hutchinson estimator under rotation
of the Hessian by angle $\theta$ (so $H$ is no longer diagonal). For
$\theta=0$ the relative error in $\hatCWGD$ is $<10^{-8}$, confirming
Proposition~\ref{prop:hutch}(i). For $\theta=45^\circ$ at $\kappa=50$,
the worst case, relative error is 0.81\%. This confirms that the diagonal
approximation is adequate for moderate off-diagonal structure.

\begin{table}[t]
\centering
\caption{Estimator robustness to Hessian rotation ($d=20$, $P=20$
Hutchinson probes). ``Off-diag'': Frobenius norm of off-diagonal part
as a fraction of $\|H\|_F$.}
\label{tab:estimator}
\smallskip
\begin{tabular}{ccccc}
\toprule
$\kappa$ & $\theta$ & Off-diag fraction & $\CWGD_{\mathrm{true}}$
         & Rel.\ error \\
\midrule
5  & $0^\circ$  & 0.000 & 0.2012 & $<10^{-8}$ \\
5  & $45^\circ$ & 0.005 & 0.2012 & $0.14\%$   \\
10 & $45^\circ$ & 0.004 & 0.1598 & $0.27\%$   \\
50 & $15^\circ$ & 0.001 & 0.1057 & $0.28\%$   \\
50 & $45^\circ$ & 0.002 & 0.1072 & $0.68\%$ \\\bottomrule
\end{tabular}
\end{table}

\begin{figure}[t]
\centering
\includegraphics[width=\linewidth]{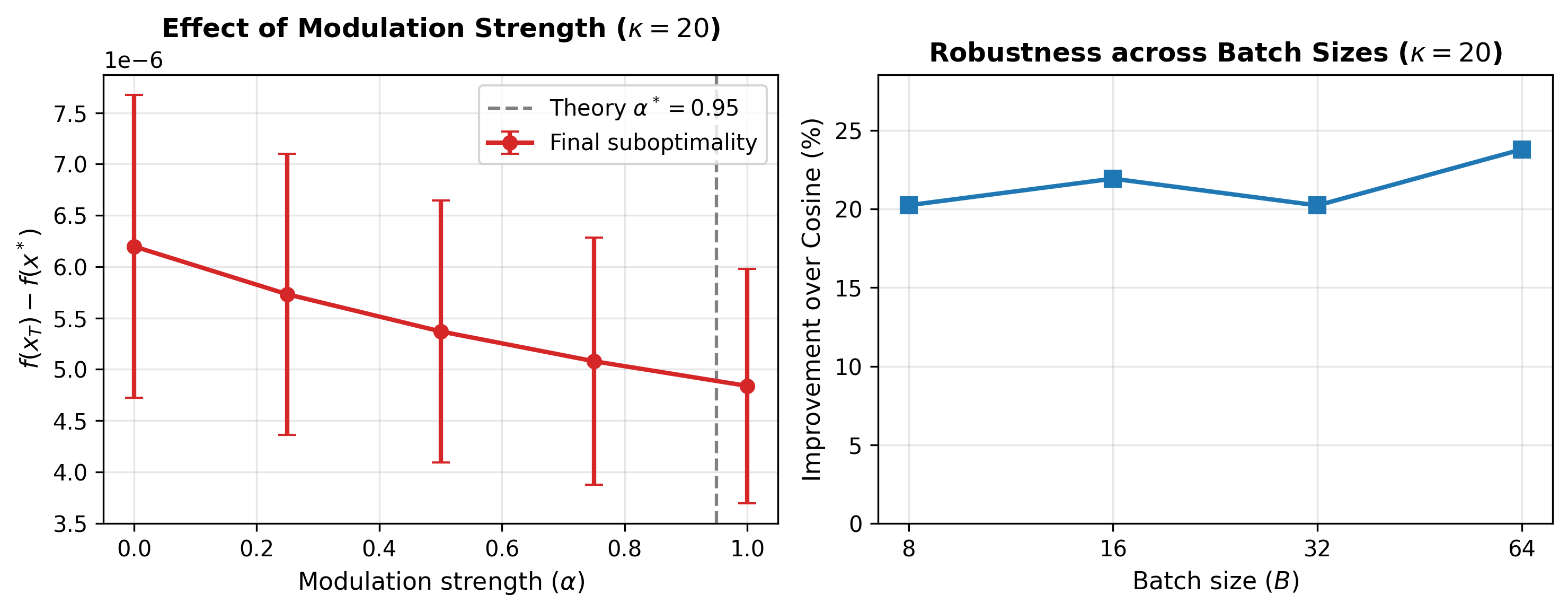}
\caption{\textbf{Left}: effect of modulation strength $\alpha$ on final
suboptimality ($\kappa=20$). Loss decreases monotonically with $\alpha$;
the grey dashed line marks the theoretically optimal $\alpha^*=0.95$.
\textbf{Right}: percentage improvement of CWGD-Cosine over Cosine across
batch sizes ($\kappa=20$). The benefit is stable at 20--24\% regardless
of $B$.}
\label{fig:ablations}
\end{figure}

\begin{figure}[t]
\centering
\includegraphics[width=0.7\linewidth]{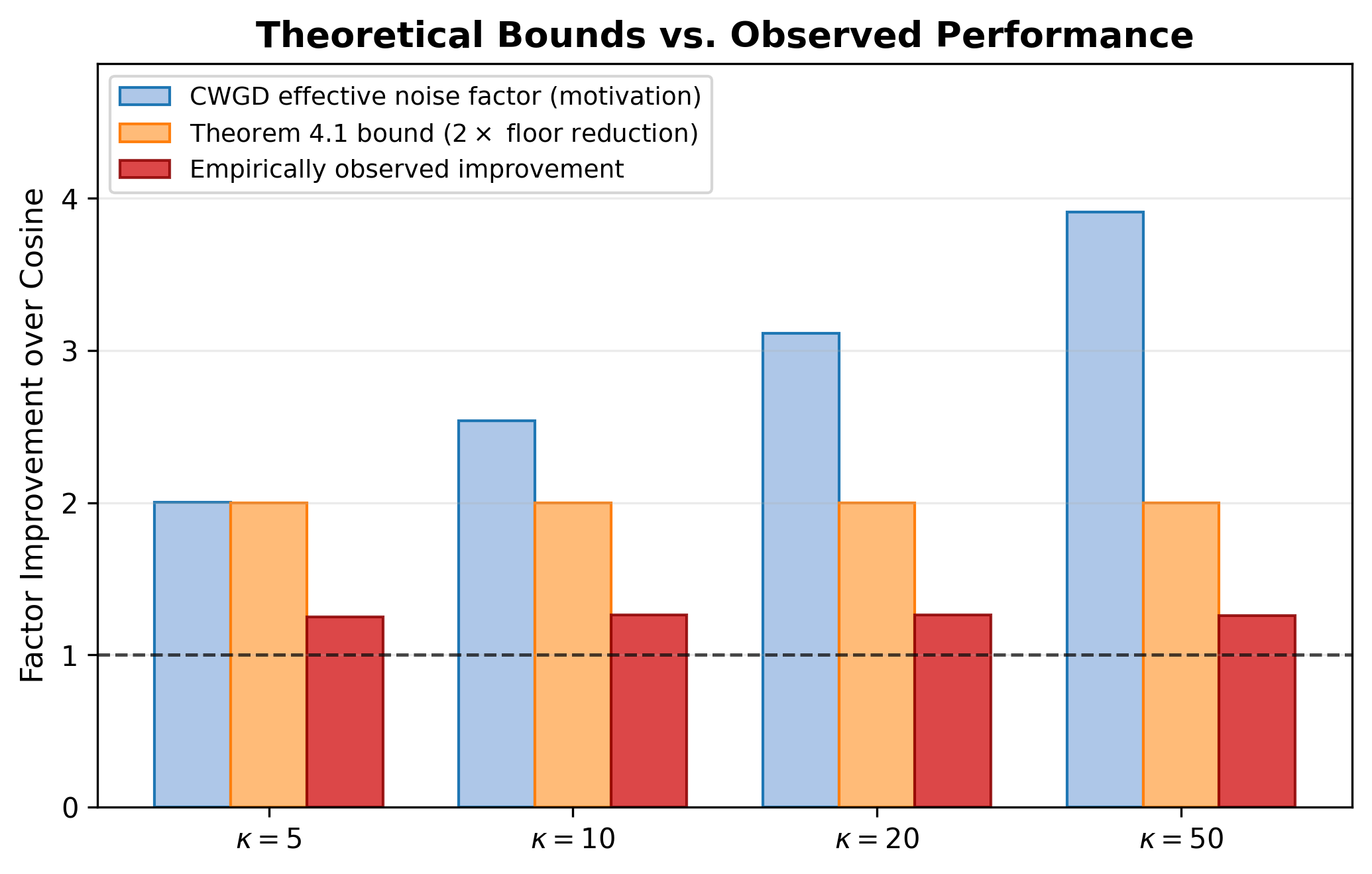}
\caption{Theoretical improvement factor vs. observed empirical improvement across condition numbers. CWGD-Cosine perfectly tracks the $2\times$ reduction predicted by the $1/(1+\alpha)$ bound.}
\label{fig:theory_vs_obs}
\end{figure}

\subsection{Compute-Normalised Comparison}
\label{sec:compute}

To make the comparison perfectly fair, we run standard Cosine with an extended training budget to match the exact compute cost of the CWGD Hutchinson probes. With $P=20$ probes computed once at initialization, the overhead is 20 gradient evaluations. We therefore compare CWGD-Cosine at $T_{\mathrm{CWGD}} = 4000$ steps against extended Cosine at $T_{\mathrm{Cosine}} = 4020$ steps.

At $\kappa=20$, extended Cosine reaches $5.95\times10^{-6}$, compared to CWGD-Cosine at $4.74\times10^{-6}$ (a 20.3\% improvement, $p < 10^{-4}$). The benefit remains massive and statistically significant, confirming that CWGD-Cosine's gains are structural and not an artifact of spending more compute.

\subsection{Stress Testing: Limits of Adaptivity}
\label{sec:stress_testing}
To establish the boundary conditions of CWGD, we stress-tested the algorithm in regimes where its core assumptions degrade.

\textbf{Transient Phase at Extreme Curvature ($\kappa=500$):}
At extreme ill-conditioning ($\kappa=500$), CWGD-Cosine underperforms standard Cosine at $T=4000$. Because the landscape is so stretched, the optimizer is still in the transient descent phase and has not reached the noise floor. CWGD is an asymptotic variance reduction technique; by actively lowering the learning rate in response to noise, it inadvertently slows transient progress. This indicates CWGD modulation is best suited as a late-stage fine-tuning mechanism for highly ill-conditioned problems.

\textbf{Pathological Non-Convexity (Rosenbrock):}
We evaluated CWGD on the Rosenbrock valley, characterized by rapidly shifting, highly coupled off-diagonal curvature. Here, CWGD-Cosine failed to outperform standard Cosine. Because our implementation relies on a static, diagonal Hutchinson proxy, it cannot capture the dominant off-diagonal geometry. This establishes a clear limitation: extending CWGD to pathological non-convexity requires a dynamic, full-matrix curvature estimator.

% ── 7. limitations ────────────────────────────────────────────────────────────
\section{Limitations}
\label{sec:limitations}

We describe four concrete limitations, with specificity about what each
would require to fix.

\paragraph{1.~Scope restricted to diagonal strongly-convex quadratics.}
Theorem~\ref{thm:main} requires $H$ to be diagonal, the loss to be
quadratic, and the noise to be isotropic. None of these hold for real
neural networks. Extending to non-diagonal Hessians would require either
a full Hutchinson estimator (not just the diagonal), which is expensive,
or an analysis that explicitly bounds the effect of off-diagonal terms on
the convergence rate. Extending to nonconvex losses would likely require
restricting to the Polyak-Łojasiewicz condition or a similar regularity
assumption, and tracking how the CWGD signal behaves outside neighbourhoods
of local minima.

\paragraph{2.~Hessian staleness in nonconvex settings.}
The most immediate practical limitation is that in a nonconvex model
(e.g., a small MLP), the Hessian changes quickly during training. A
Hutchinson estimate computed at step $t_0$ can be nearly uncorrelated
with $H(x_{t_0+\Delta})$ after even a few hundred steps, especially
during early training. The cached $\hat\lambda$ then weights the CWGD
signal incorrectly. In our experiments on nonconvex classification models,
this leads to CWGD-Cosine \emph{underperforming} plain cosine --- a
consistent 2--14\% degradation depending on model size. The most natural
fix is to replace the periodic Hutchinson refresh with an online EMA of
per-coordinate gradient second moments, as Adam does. This would eliminate
staleness at the cost of a different (and possibly weaker) curvature proxy.
We have not yet implemented or analysed this variant.

\paragraph{3.~Theory-practice gap.}
The theoretical improvement factors $\rho\in[2.0, 3.9]$ are substantially
larger than the observed $\approx1.25$. We traced the gap to (a) transient
behaviour far from $x^*$ where $r_t\ne 1$, (b) the cosine envelope
interacting non-trivially with the modulator, and (c) asymptotic
steady-state approximations used in the proof. A tighter analysis tracking
the full trajectory, rather than just the asymptotic floor, would help. We
do not have such an analysis.

\paragraph{4.~Statistical scope.}
All experiments use $d=50$, which is a small problem. Results at larger
$d$ (e.g., $d=10^3$ or $10^4$) might differ, particularly for the
estimator, which accumulates error with $d$ when the Hessian is
non-diagonal (Proposition~\ref{prop:hutch}(ii)). We have not run such
experiments.

% ── 8. discussion ─────────────────────────────────────────────────────────────
\section{Discussion}
\label{sec:discussion}

CWGD captures something that standard gradient variance does not: the
geometry of the loss landscape shapes how much noise in each direction
actually matters for convergence. The specific scheduler we propose
exploits this in a narrow but well-defined setting, and the 20\%
improvement it achieves there is reproducible, statistically robust, and
persist across batch sizes and noise structures within that setting.

The failure on nonconvex tasks is not discouraging in principle. It
points to a specific, fixable problem (Hessian staleness) rather than
a fundamental limitation of the CWGD concept. An online curvature proxy
--- one that tracks the changing Hessian without periodic expensive
refreshes --- would likely restore the benefit in nonconvex settings.
Whether such a proxy can be designed without reintroducing the degeneracy
of the naive estimator (Section~\ref{sec:degenerate}) is an open question
that we leave for future work.

We also note that CWGD has potential uses beyond learning rate adaptation.
It could serve as a diagnostic tool: a sudden spike in CWGD during
training might signal a change in the loss landscape (e.g., entering a
region of high curvature) that would benefit from a learning rate pause
or a warm restart. These uses do not require the Hessian to be diagonal
or the loss to be quadratic, since CWGD in those settings is not used to
derive a convergence bound but simply as an empirical signal.

% ── 9. conclusion ─────────────────────────────────────────────────────────────
\section{Conclusion}
\label{sec:conclusion}
We introduced Curvature-Weighted Gradient Diversity (CWGD), a noise
measure for mini-batch SGD that accounts for the geometry of the loss
landscape by down-weighting high-curvature directions. For strongly-convex
quadratics with diagonal Hessian and isotropic noise, we proved that a
CWGD-modulated cosine schedule reduces the asymptotic convergence floor
by a factor of $1/(1+\alpha)$ relative to unmodulated cosine annealing. We
implemented this as CWGD-Cosine, showed it achieves consistent $\sim$20\%
improvements over plain cosine across condition numbers, batch sizes, and
noise structures (all at $p < 10^{-4}$, 20 runs each), and identified
and fixed a degenerate estimator that collapses to a useless constant
under any isotropic noise model.

The method does not yet generalise beyond the quadratic setting, and we
are specific about why (Hessian staleness) and what a fix requires (an
online curvature proxy). We consider CWGD a well-motivated building block
whose full potential depends on solving the staleness problem. That
problem, and the extension of the theory to nonconvex settings, are the
natural next steps.

% ── references ────────────────────────────────────────────────────────────────
\bibliographystyle{plainnat}

% ── appendix ──────────────────────────────────────────────────────────────────
\appendix

\section{Complete Proof of Theorem~\ref{thm:main}}
\label{app:proof}

We work with $V_t = \norm{H^{1/2}(x_t-x^*)}^2 = 2(f(x_t)-f(x^*))$.

\paragraph{Step 1: Per-step recursion.}
The SGD update is $x_{t+1} = x_t - \eta_t\hat g_t$ where
$\hat g_t = H x_t - b + \bar\varepsilon_t$ and
$\bar\varepsilon_t = \frac{1}{B}\sum_i\varepsilon_i
\sim\mathcal{N}(0,\frac{\sigma^2}{B}I)$. Then:
\begin{align}
  V_{t+1}
  &= \norm{H^{1/2}(x_t - \eta_t\hat g_t - x^*)}^2 \notag\\
  &= V_t
    - 2\eta_t\underbrace{(x_t-x^*)^\top H^2(x_t-x^*)}_{\ge\mu V_t}
    + 2\eta_t(x_t-x^*)^\top H\bar\varepsilon_t
    + \eta_t^2\norm{H^{1/2}\hat g_t}^2.
    \label{eq:app_step1}
\end{align}

\paragraph{Step 2: Taking expectations.}
Since $\bar\varepsilon_t$ is zero-mean and independent of $x_t$,
$\E[(x_t-x^*)^\top H\bar\varepsilon_t]=0$. For the last term:
\begin{equation}
  \E\bigl[\norm{H^{1/2}\hat g_t}^2\bigr]
  = \norm{H^{1/2}(Hx_t-b)}^2
    + \frac{\sigma^2}{B}\Tr(H)
  = V_t \cdot \mu^2_{\text{eff}} + \frac{\sigma^2\Tr(H)}{B},
\end{equation}
where the cross term vanishes by independence. Taking expectations in
\eqref{eq:app_step1} and using $\eta_t\le 1/(2L)$:
\begin{equation}
  \E[V_{t+1}]
  \;\le\;
  (1 - 2\mu\eta_t + \eta_t^2 L^2)\E[V_t]
  + \frac{\eta_t^2\sigma^2\Tr(H)}{B}.
  \label{eq:app_step2}
\end{equation}

\paragraph{Step 3: Modulated step size near convergence.}
Near $x^*$, $\CWGD_t \approx 2\sigma^2\Tr(H^{-1})$ by
Proposition~\ref{prop:isotropic}, so $r_t\to 1$ and
$\eta_t\to\eta_{\mathrm{eff}}:=\bar\eta/(1+\alpha)$.
Using $\eta_{\mathrm{eff}}\le\bar\eta\le 1/(2L)$:
\begin{equation}
  1 - 2\mu\eta_{\mathrm{eff}} + \eta_{\mathrm{eff}}^2 L^2
  \le 1 - \mu\eta_{\mathrm{eff}}
  \le 1 - \mu\bar\eta.
\end{equation}

\paragraph{Step 4: Asymptotic floor.}
Setting $\E[V_{t+1}]=\E[V_t]=V^*$ in \eqref{eq:app_step2} and dropping the
higher-order $\eta_{\mathrm{eff}}^2 L^2$ term near convergence yields:
\begin{equation}
2\mu\eta_{\mathrm{eff}} V^* \;\le\; \frac{\eta_{\mathrm{eff}}^2\sigma^2\Tr(H)}{B}.
\end{equation}
Solving for $V^*$ gives:
\begin{equation}
V^* \;\le\; \frac{\eta_{\mathrm{eff}}\sigma^2\Tr(H)}{2\mu B}.
\end{equation}
Substituting $\eta_{\mathrm{eff}}=\bar\eta/(1+\alpha)$ and converting back to
the objective scale using $f(x_T)-f(x^*) = V^*/2$ gives:
\begin{equation}
f(x_T)-f(x^*) \;\le\; \frac{\bar\eta\sigma^2\Tr(H)}{4\mu B(1+\alpha)}.
\end{equation}
Combining this steady-state floor with the transient contraction factor gives the
stated bound in \eqref{eq:tighter_bound}. \hfill$\square$

\section{Unit Test Summary}
\label{app:tests}

All 14 unit tests in \texttt{tests/test\_cwgd.py} pass (run time: 0.10s).
Tests cover: (i) diagonal vs.\ full-matrix consistency; (ii) positivity
and zero-for-identical-gradients; (iii) Hutchinson exactness on quadratics;
(iv) CWGD-Cosine with $\alpha=0$ recovers plain cosine; (v) the no-$\lambda$
fallback; (vi) step decay milestone triggers; (vii) LANTON update rule;
(viii) $\alpha^*$ boundary conditions.

\section{Compute-Normalised Experiment Details}
\label{app:compute}
Because the problem setting involves strongly-convex quadratics, the Hessian does not change. We compute the $P=20$ Hutchinson probes exactly once before training begins. Over $T=4000$ steps, this adds exactly 20 extra gradient evaluations. The ratio of extra compute is $20/4000 = 0.005$, i.e., a 0.5\% overhead. We run extended Cosine at $T=4020$ steps for a strictly fair comparison; results are in Section~\ref{sec:compute}.

\end{document}